\documentclass{article}

    \PassOptionsToPackage{numbers, compress}{natbib}
 \usepackage[preprint]{neurips_2025}

\usepackage[utf8]{inputenc} 
\usepackage[T1]{fontenc}    
\usepackage{hyperref}       
\usepackage{url}            
\usepackage{booktabs}       
\usepackage{amsfonts}       
\usepackage{nicefrac}       
\usepackage{microtype}      
\usepackage{xcolor}         

\usepackage{float}

\usepackage{mathtools}
\newcommand{\E}{\mathbb{E}}
\newcommand{\R}{\mathbb{R}}
\newcommand{\N}{\mathbb{N}}

\newcommand{\calP}{\mathcal{P}}

\newcommand{\calI}{\mathcal{I}}
\newcommand{\calT}{\mathcal{T}}

\def\calN{{\mathcal{N}}}

\newcommand{\W}{\mathcal{W}}
\newcommand{\AW}{\mathcal{AW}}

\newcommand{\cpl}{\mathrm{Cpl}}
\newcommand{\bccpl}{\mathrm{Cpl}_{\mathrm{bc}}}

\newcommand{\tr}{\mathrm{tr}}
\newcommand{\diag}{\mathrm{diag}}

\newcommand{\avar}{\mathrm{AVaR}}

\usepackage{amsmath, amssymb, amsthm}
\newtheorem{theorem}{Theorem}

\newtheorem{definition}{Definition}
\newtheorem{proposition}[theorem]{Proposition}

\theoremstyle{remark}

\title{Nested Optimal Transport Distances}
\author{%
Ruben Bontorno\\
  Department of Mathematics\\
  ETH Z\"{u}rich\\
  \texttt{rbontorno@student.ethz.ch} \\
  \And
  Songyan Hou\\
  Department of Mathematics\\
  ETH Z\"{u}rich\\
  \texttt{songyan.hou@math.ethz.ch} \\
}

\begin{document}

\maketitle

\begin{abstract}
Simulating realistic financial time series is essential for stress testing, scenario generation, and decision-making under uncertainty.  
Despite advances in deep generative models, there is no consensus metric for their evaluation.  
We focus on generative AI for financial time series in decision-making applications and employ the \emph{nested optimal transport distance}—a time-causal variant of optimal transport distance, which is robust to tasks such as hedging, optimal stopping, and reinforcement learning. Moreover, we propose a statistically consistent, naturally parallelizable algorithm for its computation, achieving substantial speedups over existing approaches.
\end{abstract}

\section{Introduction}
\label{sec:intro}
Recent advances in generative AI have opened new frontiers for applications in finance, such as financial modeling, stress testing, scenario generation, automated financial services, and decision-making under uncertainty; see \cite{assefa2020generating, ericson2024deep}. The goal of generative AI is to generate synthetic financial data ``indistinguishable'' from real financial data, and the ``indistinguishability'' are evaluated by probabilistic metrics. Due to the unique characteristics in different applications of generated financial data, it is natural to expect no consensus metrics like FID (Fr\'{e}chet Inception Distance) in image generation evaluation. But the lack of even an intra-application consensus-metric makes it difficult to compare different models \cite{rengim2025synthetic}. A qualified consensus-metric is supposed to satisfy the following two criteria: robustness and tractability.

\textbf{Robustness:} The consensus-metric should robustly control a basket of quantities of interest.

\textbf{Tractability:} The metric should be computationally tractable with samples.

A large class of generative AI models for financial time series is trained for decision-making applications, including dynamic hedging, optimal stopping, utility maximization, and reinforcement learning. Notably, these problems are not continuous with respect to widely-used distances, such as the Maximum Mean Discrepancy (MMD)  and the Wasserstein distances ($\W$-distances). However, these problems are Lipschitz-continuous with respect to stronger metrics, called adapted Wasserstein distances ($\AW$-distances), also known as nested Wasserstein distances, which have been introduced as a variant of $\W$-distances that accounts for the time-causal structure of stochastic processes. $\AW$-distances serve as robust distances for many dynamic stochastic optimization problems across several fields, particularly in mathematical finance and economics; see \cite{backhoff2020adapted, bion2009time, glanzer2019incorporating, bartl2023sensitivity}. Therefore, for robustness purposes, it is desirable to evaluate the generated time series distribution under $\AW$-distances.

While $\AW$-distances ensure robustness, the existing implementation calculating $\AW$-distances are slow and does not scale to long time series (see implementations in \cite{pichler2022nested,eckstein2024computational,bayraktar2023fitted}). In this paper, we propose a natural parallelization algorithm to calculate $\AW$-distances, which achieves orders-of-magnitude speedups over existing implementations and also behaves statistically consistently.

\section{Nested optimal transport distances}
\label{sec:not}
We regard $\R^{dT}$ as the space of $d$-dimensional discrete-time paths with $T$ time steps. Notation follows standard conventions and is intended to be self-explanatory; further details or clarifications of notations are provided in Appendix \ref{sec:notation}.
\begin{definition}
    For $\mu,\nu\in\calP_2(\R^{dT})$, the \textit{Wasserstein-2 distance} $\W_2(\cdot,\cdot)$ on $\mathcal{P}_2(\R^{dT})$ is defined by 
    \begin{equation}
        \W_2^2(\mu,\nu) \coloneqq \inf_{\pi \in \cpl(\mu,\nu)}\int \| x - y\|^2 \,\pi(dx,dy),
    \end{equation}
    where $\cpl(\mu,\nu)$ denotes the set of couplings between $\mu$ and $\nu$, that is, probabilities in $\mathcal{P}(\R^{dT}\times \R^{dT})$ with first marginal $\mu$ and second marginal $\nu$. 
\end{definition}
Next, we restrict our attention to couplings $\pi \in \cpl(\mu,\nu)$ such that the conditional law of $\pi$ is still a coupling of the conditional laws of $\mu$ and $\nu$, that is for all $t=0,\dots,T-1$, $\pi_{x_{1:t},y_{1:t}} \in \cpl\big(\mu_{x_{1:t}}, \nu_{y_{1:t}}\big)$. Such couplings are called bi-causal, and denoted by $\bccpl(\mu,\nu)$.
\begin{definition}
     For $\mu,\nu\in\calP_2(\R^{dT})$, the \textit{adapted Wasserstein-2 distance} $\AW_2(\cdot,\cdot)$ on $\mathcal{P}(\R^{dT})$ is defined by 
    \begin{equation}
    \label{eq:adaptedwd} 
        \AW_2^2(\mu,\nu) \coloneqq \inf_{\pi \in \bccpl(\mu,\nu)}\int\sum_{t=1}^T\|x_t-y_t\|^2\,\pi(dx,dy).
    \end{equation}
\end{definition}
The adaptedness (or bi-causality) imposed on couplings modifies the Wasserstein distance to ensure robustness of a large class of ``well-defined" dynamic optimization problems, ranging from optimal stopping and utility maximization to dynamic risk minimization and dynamic hedging; see
\cite{bion2009time, glanzer2019incorporating, bartl2023sensitivity, pflug2014multistage, acciaio2020causal, acciaio2024time}. We therefore present only the general statement here, while providing illustrative examples in the appendix and referring the reader to the cited references for further details. 

\textbf{Robustness (\cite{bartl2023sensitivity, pflug2014multistage, acciaio2020causal, acciaio2024time}). } Let $\mu, \nu \in \calP_2(\R^{dT})$ and $v\colon \calP_2(\R^{dT}) \to \R$, where $v(\mu)$ denotes the optimal value of a ``well-defined'' dynamic optimization problem under $\mu$. Then under ``mild conditions'', there exists $L \geq 0$ such that the following holds 
\begin{equation*}
    \vert v(\mu) - v(\nu) \vert \leq L\,\AW_2(\mu,\nu).
\end{equation*}

From the perspective of nested disintegration, Pflug-Pichler define the adapted Wasserstein distance as nested distance in \cite{pflug2014multistage} and establish an alternative representation of $\AW_2(\cdot,\cdot)$ through the dynamic programming principle; see the proof of Proposition~\ref{prop.dpp} in Appendix~\ref{subsec:proof}.

\begin{proposition}[Dynamic programming principle]
\label{prop.dpp}
Let $\mu,\nu\in\calP_2(\R^{dT})$. Set $V_T^{\mu,\nu} \equiv 0$ and define for all $t = 0,\ldots, T-1$,
\[
V_t^{\mu,\nu}(x_{1:t},y_{1:t}) = \inf_{\pi_{x_{1:t},y_{1:t}}^{t+1}\in \bccpl(\mu_{x_{1:t}}^{t+1},\nu_{y_{1:t}}^{t+1})}\int \Big[\Vert x_{t+1} - y_{t+1}\Vert^2 + V^{\mu,\nu}_{t+1}(x_{1:t+1},y_{1:t+1})\Big]d\pi_{x_{1:t},y_{1:t}}^{t+1}.
\]
Then for all $t = 0,\ldots, T-1$, $ V_t^{\mu,\nu}(x_{1:t},y_{1:t}) = \AW^2_{2}(\mu_{x_{1:t}},\nu_{y_{1:t}})$ and $V_0^{\mu,\nu} \equiv \AW^2_{2}(\mu,\nu)$.
\end{proposition}
The computation of \(V_t^{\mu, \nu}(x_{1:t}, y_{1:t})\) can be fully parallelized over all admissible pairs \((x_{1:t}, y_{1:t})\), which is the key reason why our algorithm is naturally parallel. Before detailing our algorithm, we need to introduce \emph{adapted empirical measures}, originally proposed in \cite{backhoff2022estimating} to address the statistical estimation problem for $\AW$-distances. 

\begin{definition}
    Let $N \in \N$ be the number of samples and $\Delta_N > 0$. We tile $\R^{dT}$ with cubes with edges of length $\Delta_N$ and let $\varphi^N$ map each such small cube to its center. For $\mu \in \calP_1(\R^{dT})$, we let $(x^{(n)})_{n\in\N}$ be i.i.d. samples from $\mu$. Then we define $\hat{\mu}^N = \sum_{i=1}^{N}\delta_{\varphi^N(x^{(i)})}$
    and call $\hat{\mu}^N$ the adapted empirical measures of $\mu$. 
\end{definition}
 For $\AW$-distances, empirical measures fail to converge, but adapted empirical measures converge to the underlying measure as the number of samples $N$ goes to infinity, which provides statistical guarantee approximating $\mu$ with $\hat \mu^N$ in calculating $\AW$-distances; see the proof of Theorem~\ref{thm:stat} in Appendix~\ref{subsec:proof}.
\begin{theorem}
\label{thm:stat}
    Let $\mu \in \calP_2(\R^{dT})$ and $\Delta_N  = N^{-\frac{1}{dT}}$. Then $\lim_{N\to\infty}\AW_2(\mu,\hat{\mu}^N) = 0$ almost surely.
\end{theorem}

\textbf{Parallel computing algorithm:} 
We now introduce our algorithm, which consists of two main steps: \emph{quantization} and \emph{backward computation}. In the first quantization step, we fix \(\Delta_N\) and map each sample \(x^{(i)} \in \mathbb{R}^{dT}\) to its quantized counterpart \(q^{(i)} \coloneqq \varphi^N(x^{(i)})\) on lattice. Since the quantized samples lie on a lattice in \(\mathbb{R}^{dT}\), if $N$ is sufficiently large, multiple quantized samples coincide up to time \(t\), hereby endowing \(\hat{\mu}^N\) a tree structure; see visualization in Figure~\ref{fig:realfake}. In the second backward computation step, we first leverage this tree structure and compute the conditional distribution:
\[
(\hat{\mu}^N)_{q_{1:t}}^{t+1} = \frac{1}{|\calI_{q_{1:t}}|}\sum_{i \in \calI_{q_{1:t}}}\delta_{q_{t+1}^{(i)}}, \quad \calI_{q_{1:t}} = \{i \leq N \mid q_{1:t}^{(i)} = q_{1:t}\}, \quad q_{1:t} \in \{q_{1:t}^{(i)} \mid i\leq N\}.
\]  
Once \((\hat{\mu}^N)_{q_{1:t}}^{t+1}\) and \((\hat{\nu}^N)_{q_{1:t}}^{t+1}\) have been computed, we can apply the dynamic programming principle to compute \(V_t^{\hat{\mu}^N, \hat{\nu}^N}\) backward, continuing until we reach $V_0^{\hat{\mu}^N, \hat{\nu}^N}$, which equals $\AW^2_{2}(\hat{\mu}^N, \hat{\nu}^N)$ by Proposition~\ref{prop.dpp}. Finally, by Theorem~\ref{thm:stat}, $\AW^2_{2}(\hat{\mu}^N, \hat{\nu}^N)$ converges to $\AW^2_{2}(\mu,\nu)$.

\textbf{Markovian improvement: } In general, the moment convergence rate of $\AW_2$: $\E[\AW_2(\mu, \hat \mu^N)] = O(N^{-\frac{1}{dT}})$; see \cite[Theorem~1.5]{backhoff2022estimating} suffers from the same curse of dimensionality as $\W_2$; see \cite[Theorem~1]{fournier2015rate}. This significantly restricts its applicability to long time series, even in the univariate case. However, if $\mu$ is known to be Markovian, there is a natural Markovian modification of the adapted empirical measure, simply replacing $(\hat{\mu}^N)_{q_{1:t}}^{t+1}$ by $(\hat{\mu}^N)_{q_{t}}^{t+1}$, which can be similarly defined as $(\hat{\mu}^N)_{q_{1:t}}^{t+1}$; see \cite{backhoff2022estimating}. This modification yields substantially improved statistical convergence rates $O(N^{-\frac{1}{2d}})$ that are independent of $T$; see Theorem~6.1 in \cite{backhoff2022estimating}. In addition, the Markovian modification significantly reduces the computational time in practice.

\section{Numerical experiments}
\label{sec:numerical}

For Gaussian distributions, explicit closed-form expressions exist for both the $\AW_2$ and $\W_2$ distances (see Appendix~\ref{subsec:close_form_gaussian}). We assess the convergence of our algorithm by benchmarking against these closed forms. All experiments were conducted on an AMD EPYC 7763 processor (64 cores, 2.45 GHz).

\textbf{Ornstein-Uhlenbeck processes.} Let us first consider the Ornstein-Uhlenbeck (OU) processes $X^\sigma = (X^\sigma_t)_{t\geq 0}$, defined by 
\[
dX^\sigma_t =  - X^\sigma_tdt+\sigma dW_t, \quad X_0 = 0,
\]
where $(W_t)_{t\geq 0}$ is a Brownian motion and $\sigma \in \R_{+}$. Let $N = 5$, $\Delta t = 1/N$, and consider $\tilde X^\sigma_{1:N} \coloneqq[X^\sigma_{\Delta t}, \dots, X^\sigma_{N\Delta t}]$. Denoting by $\mu^\sigma$ the distribution of $\tilde X^\sigma_{1:N}$, we compute $\AW_2^2(\mu^1,\mu^3)$ with randomly generated i.i.d. samples using the Markovian implementation of our algorithm, and evidently, numerical values of $\AW_2^2$ converge to its theoretical value; see Figure~\ref{fig:ou}.
\begin{figure}[H]
    \centering
    \includegraphics[width=\linewidth]{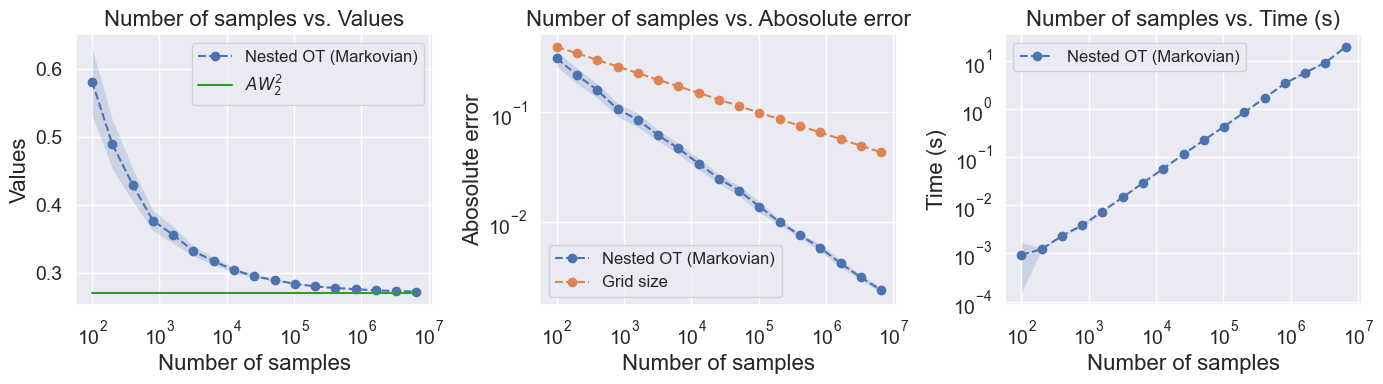}
    \caption{Numerical optimal values, absolute errors, and runtimes for $\AW_2^2$ with Markovian implementation across different sample sizes (averaged over 10 independent runs per sample size).}   
    \label{fig:ou}
\end{figure}

\textbf{Fake Brownian Motion.} Let $(W_t)_{t\in[0,1]}$ be a Brownian motion, $Z \sim \calN(0,1)$ independent of $(W_t)_{t\in[0,1]}$ and $\delta \in (0,1)$. Now, we define a \emph{fake Brownian motion} $X = (X_t)_{t\in[0,1]}$ such that
\[
X_t = \begin{cases}
    W_t + \frac{t - \delta}{1 - \delta} (Z - W_1) + \frac{1 - t}{1 - \delta} (\sqrt{\delta} Z - W_\delta), \quad &t \in [\delta, 1], \\
    \frac{t}{\sqrt \delta} Z,\quad  &t \in [0,\delta]. \\
\end{cases}
\]

\begin{figure}[H]
    \centering
    \includegraphics[width=0.9\linewidth]{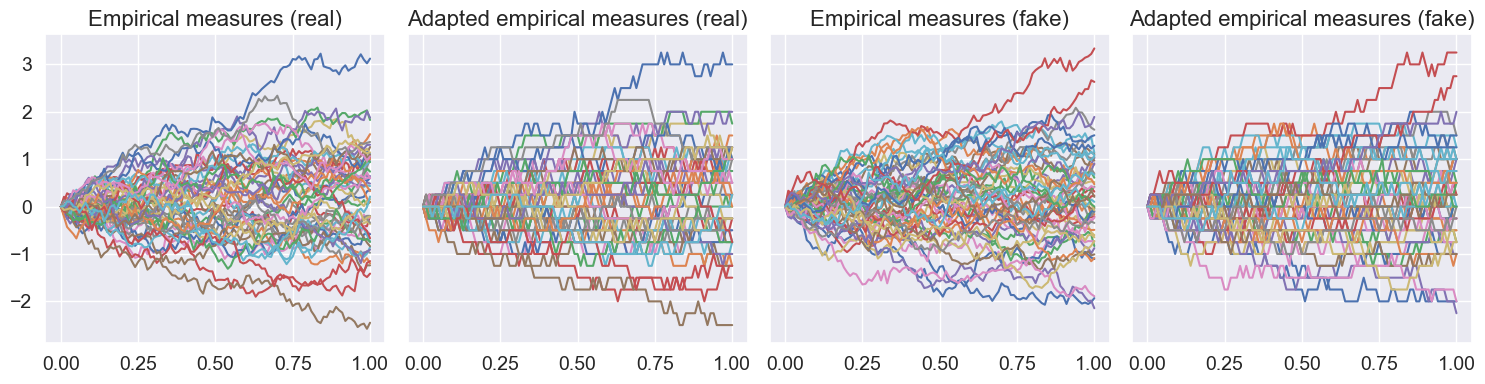}
    \caption{Visualization of empirical measures and adapted empirical measures with $50$ samples for real Brownian motion $(W_t)_{t\geq 0}$ and fake Brownian motion $(X_t)_{t\geq 0}$.}
    \label{fig:realfake}
\end{figure}

The process \(X\) is called ``fake" because, for small \(\delta\), its sample paths are visually indistinguishable from those of Brownian motion; see visualization in Figure~\ref{fig:realfake}.  
However, \(X\) is not a Brownian motion since it satisfies \(X_\delta = \sqrt{\delta}\, X_1\), meaning that \(X_1\) is perfectly predictable at time \(\delta\) from the observation of \(X_\delta\).  
If we interpret \(X\) and \(W\) as asset price processes, they exhibit fundamentally different behaviors:  
the simple strategy of going long on \(X\) when \(X_\delta > 0\) and short on \(X\) when \(X_\delta < 0\) yields arbitrage opportunities, whereas trading solely in \(W\) is arbitrage-free.

Consider the marginal $X_{\delta, t, 1} \coloneqq [X_\delta, X_t, X_1]^\top$, where $0 < \delta < t < 1$, which satisfies that
\[
X_{\delta, t, 1} = L_{\delta, t} 
\begin{bmatrix}
    Z\\
    \frac{W_t}{\sqrt{t}}\\
    \frac{W_1 - W_t}{\sqrt{1-t}}
\end{bmatrix},\quad L_{\delta, t} 
= \begin{bmatrix}
    \sqrt{\delta} & 0 & 0\\
    \frac{t-\delta + (1-t)\sqrt{\delta}}{1-\delta} & \frac{(1-t)(t-\delta)}{\sqrt{t}(1-\delta)} & 0\\
    1 & 0 & 0
\end{bmatrix}.
\]
We denote by \(\mu^X_{\delta,t,1}\) the distribution of \(X_{\delta, t, 1}\) and by \(\mu^W_{\delta,t,1}\) the distribution of \(W_{\delta, t, 1} \coloneqq [W_\delta,\, W_t,\, W_1]^\top\). For \(\delta = 0.1\) and \(t = 0.5\), we compute \(\AW_2^2(\mu^X_{\delta,t,1}, \mu^W_{\delta,t,1})\) and \(\W_2^2(\mu^X_{\delta,t,1}, \mu^W_{\delta,t,1})\) numerically, and compare them with their corresponding theoretical values, calculated by \eqref{eq:W} and \eqref{eq:AW}; see Figure~\ref{fig:fakebm}. Evidently, numerical values of $\AW_2^2$ and $\W_2^2$ converge to their theoretical values, respectively, but $\mu^X_{\delta,t,1}$ and $ \mu^W_{\delta,t,1}$ are only well-distinguished by $\AW_2^2$ and not by $\W_2^2$. The values of $\AW_2^2$ are computed with non-Markovian implementation of our algorithm with \href{https://github.com/justinhou95/NestedOT}{PNOT: Python Nested Optimal Transport} since $X_{\delta,t,1}$ is not Markovian, and those of $\W_2^2$ are computed with POT: Python Optimal Transport \cite{flamary2021pot}.  

\begin{figure}[H]
    \centering
    \includegraphics[width=0.9\linewidth]{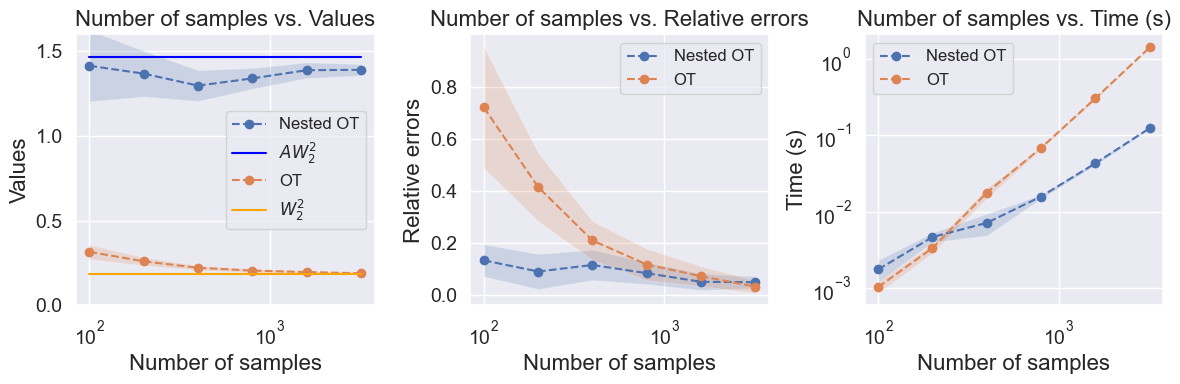}
    \caption{Numerical optimal values, absolute errors, and runtimes for $\AW_2^2$ and $\W_2^2$ across different sample sizes (averaged over 10 independent runs per sample size). The $\AW_2^2$ is calculated with our algorithm with non-Markovian implementation and the $\W_2^2$ is calculated by \texttt{ot.lp.emd} (POT).}   
    \label{fig:fakebm}
\end{figure}

\section{Conclusion}
In this work, we highlight the lack of a consensus metric for evaluating generative AI in financial time series and propose the nested optimal transport. To address the computational challenges of this metric, we develop a naturally parallelizable algorithm that enables substantial speedups over existing methods. Our approach is theoretically supported by statistical analysis and empirically validated through numerical experiments.

\bibliographystyle{unsrtnat}
\bibliography{references.bib}

\begin{thebibliography}{22}
\providecommand{\natexlab}[1]{#1}
\providecommand{\url}[1]{\texttt{#1}}
\expandafter\ifx\csname urlstyle\endcsname\relax
  \providecommand{\doi}[1]{doi: #1}\else
  \providecommand{\doi}{doi: \begingroup \urlstyle{rm}\Url}\fi

\bibitem[Assefa et~al.(2020)Assefa, Dervovic, Mahfouz, Tillman, Reddy, and
  Veloso]{assefa2020generating}
Samuel~A Assefa, Danial Dervovic, Mahmoud Mahfouz, Robert~E Tillman, Prashant
  Reddy, and Manuela Veloso.
\newblock Generating synthetic data in finance: opportunities, challenges and
  pitfalls.
\newblock In \emph{Proceedings of the First ACM International Conference on AI
  in Finance}, pages 1--8, 2020.

\bibitem[Ericson et~al.(2024)Ericson, Zhu, Han, Fu, Li, Guo, and
  Hu]{ericson2024deep}
Lars Ericson, Xuejun Zhu, Xusi Han, Rao Fu, Shuang Li, Steve Guo, and Ping Hu.
\newblock Deep generative modeling for financial time series with application
  in var: A comparative review.
\newblock \emph{arXiv preprint arXiv:2401.10370}, 2024.

\bibitem[Rengim~Cetingoz and Lehalle(2025)]{rengim2025synthetic}
Adil Rengim~Cetingoz and Charles-Albert Lehalle.
\newblock Synthetic data for portfolios: A throw of the dice will never abolish
  chance.
\newblock \emph{arXiv e-prints}, pages arXiv--2501, 2025.

\bibitem[Backhoff-Veraguas et~al.(2020)Backhoff-Veraguas, Bartl, Beiglb{\"o}ck,
  and Eder]{backhoff2020adapted}
Julio Backhoff-Veraguas, Daniel Bartl, Mathias Beiglb{\"o}ck, and Manu Eder.
\newblock Adapted wasserstein distances and stability in mathematical finance.
\newblock \emph{Finance and Stochastics}, 24:\penalty0 601--632, 2020.

\bibitem[Bion-Nadal(2009)]{bion2009time}
Jocelyne Bion-Nadal.
\newblock Time consistent dynamic risk processes.
\newblock \emph{Stochastic Processes and their Applications}, 119\penalty0
  (2):\penalty0 633--654, 2009.

\bibitem[Glanzer et~al.(2019)Glanzer, Pflug, and
  Pichler]{glanzer2019incorporating}
Martin Glanzer, Georg~Ch Pflug, and Alois Pichler.
\newblock Incorporating statistical model error into the calculation of
  acceptability prices of contingent claims.
\newblock \emph{Mathematical Programming}, 174\penalty0 (1):\penalty0 499--524,
  2019.

\bibitem[Bartl and Wiesel(2023)]{bartl2023sensitivity}
Daniel Bartl and Johannes Wiesel.
\newblock Sensitivity of multiperiod optimization problems with respect to the
  adapted wasserstein distance.
\newblock \emph{SIAM Journal on Financial Mathematics}, 14\penalty0
  (2):\penalty0 704--720, 2023.

\bibitem[Pichler and Weinhardt(2022)]{pichler2022nested}
Alois Pichler and Michael Weinhardt.
\newblock The nested sinkhorn divergence to learn the nested distance.
\newblock \emph{Computational Management Science}, 19\penalty0 (2):\penalty0
  269--293, 2022.

\bibitem[Eckstein and Pammer(2024)]{eckstein2024computational}
Stephan Eckstein and Gudmund Pammer.
\newblock Computational methods for adapted optimal transport.
\newblock \emph{The Annals of Applied Probability}, 34\penalty0 (1A):\penalty0
  675--713, 2024.

\bibitem[Bayraktar and Han(2023)]{bayraktar2023fitted}
Erhan Bayraktar and Bingyan Han.
\newblock Fitted value iteration methods for bicausal optimal transport.
\newblock \emph{arXiv preprint arXiv:2306.12658}, 2023.

\bibitem[Pflug and Pichler(2014)]{pflug2014multistage}
Georg~Ch Pflug and Alois Pichler.
\newblock \emph{Multistage stochastic optimization}, volume 1104.
\newblock Springer, 2014.

\bibitem[Acciaio et~al.(2020)Acciaio, Backhoff-Veraguas, and
  Zalashko]{acciaio2020causal}
Beatrice Acciaio, Julio Backhoff-Veraguas, and Anastasiia Zalashko.
\newblock Causal optimal transport and its links to enlargement of filtrations
  and continuous-time stochastic optimization.
\newblock \emph{Stochastic Processes and their Applications}, 130\penalty0
  (5):\penalty0 2918--2953, 2020.

\bibitem[Acciaio et~al.(2024{\natexlab{a}})Acciaio, Eckstein, and
  Hou]{acciaio2024time}
Beatrice Acciaio, Stephan Eckstein, and Songyan Hou.
\newblock Time-causal vae: Robust financial time series generator.
\newblock \emph{arXiv preprint arXiv:2411.02947}, 2024{\natexlab{a}}.

\bibitem[Backhoff et~al.(2022)Backhoff, Bartl, Beiglb{\"o}ck, and
  Wiesel]{backhoff2022estimating}
Julio Backhoff, Daniel Bartl, Mathias Beiglb{\"o}ck, and Johannes Wiesel.
\newblock Estimating processes in adapted wasserstein distance.
\newblock \emph{The Annals of Applied Probability}, 32\penalty0 (1):\penalty0
  529--550, 2022.

\bibitem[Fournier and Guillin(2015)]{fournier2015rate}
Nicolas Fournier and Arnaud Guillin.
\newblock On the rate of convergence in wasserstein distance of the empirical
  measure.
\newblock \emph{Probability Theory and Related Fields}, 162\penalty0
  (3):\penalty0 707--738, 2015.

\bibitem[Flamary et~al.(2021)Flamary, Courty, Gramfort, Alaya, Boisbunon,
  Chambon, Chapel, Corenflos, Fatras, Fournier, et~al.]{flamary2021pot}
R{\'e}mi Flamary, Nicolas Courty, Alexandre Gramfort, Mokhtar~Z Alaya,
  Aur{\'e}lie Boisbunon, Stanislas Chambon, Laetitia Chapel, Adrien Corenflos,
  Kilian Fatras, Nemo Fournier, et~al.
\newblock Pot: Python optimal transport.
\newblock \emph{Journal of Machine Learning Research}, 22\penalty0
  (78):\penalty0 1--8, 2021.

\bibitem[Backhoff-Veraguas et~al.(2017)Backhoff-Veraguas, Beiglbock, Lin, and
  Zalashko]{backhoff2017causal}
Julio Backhoff-Veraguas, Mathias Beiglbock, Yiqing Lin, and Anastasiia
  Zalashko.
\newblock Causal transport in discrete time and applications.
\newblock \emph{SIAM Journal on Optimization}, 27\penalty0 (4):\penalty0
  2528--2562, 2017.

\bibitem[Acciaio and Hou(2024)]{acciaio2024convergence}
Beatrice Acciaio and Songyan Hou.
\newblock Convergence of adapted empirical measures on $\mathbb{R}^d$.
\newblock \emph{The Annals of Applied Probability}, 34\penalty0 (5):\penalty0
  4799--4835, 2024.

\bibitem[Takatsu(2010)]{takatsu2011wasserstein}
Asuka Takatsu.
\newblock On wasserstein geometry of gaussian measures.
\newblock \emph{Probabilistic approach to geometry}, 57:\penalty0 463--472,
  2010.

\bibitem[Gunasingam and Leonard~Wong(2025)]{gunasingam2025adapted}
Madhu Gunasingam and Ting-Kam Leonard~Wong.
\newblock Adapted optimal transport between gaussian processes in discrete
  time.
\newblock \emph{Electronic Communications in Probability}, 30:\penalty0 1--14,
  2025.

\bibitem[Acciaio et~al.(2024{\natexlab{b}})Acciaio, Hou, and
  Pammer]{acciaio2024entropic}
Beatrice Acciaio, Songyan Hou, and Gudmund Pammer.
\newblock Entropic adapted wasserstein distance on gaussians.
\newblock \emph{arXiv preprint arXiv:2412.18794}, 2024{\natexlab{b}}.

\bibitem[Bartl et~al.(2024)Bartl, Beiglb{\"o}ck, and Pammer]{bartl2021WSSP}
Daniel Bartl, Mathias Beiglb{\"o}ck, and Gudmund Pammer.
\newblock The wasserstein space of stochastic processes.
\newblock \emph{Journal of the European Mathematical Society}, 2024.

\end{thebibliography}


\appendix

\section{Technical Appendices and Supplementary Material}
\subsection{Notations}
\label{sec:notation}

We regard $\R^{dT}$ as the space of $d$-dimensional discrete-time paths with $T$ time steps, $x = (x_1,\dots,x_T)$, equipped with the Euclidean norm $\Vert \cdot \Vert$.
To refer to the subpath $(x_s,\dots,x_t)$ of $x$ we also write $x_{s:t}$. Moreover, for $\mu\in\mathcal{P}(\R^{dT})$, we denote the $t$-th marginal of $\mu$ by $\mu_t$, the up-to-time-$t$ marginal of $\mu$ by $\mu_{1:t}$, and the conditional law of $\mu$ conditional on $x_{1:t}$ by $\mu_{x_{1:t}}$, so that the following disintegration holds: $\mu(dx_{1:T}) = \mu_{1:t}(dx_{1:t}) \circ \mu_{x_{1:t}}(dx_{t+1:T})$. To denote the $x_s$-marginal of $\mu_{x_{1:t}}$ for $s>t$, we use the notation $\mu_{x_{1:t}}^s(dx_s) = (x_{t+1:T} \mapsto x_s)_\# \mu_{x_{1:t}}(dx_s)$, where $\#$ designates the push-forward operation.
For notational completeness, we let $\mu_{x_{1:0}} = \mu$ where $x_{1:0}$ is a dummy variable. For $\mu,\nu \in\mathcal{P}(\R^{dT})$, we denote their set of couplings by $\cpl(\mu,\nu)=\{\pi\in\mathcal{P}(\R^{dT}\times\R^{dT}) : \pi(dx \times \R^{dT}) = \mu(dx), \pi(\R^{dT}\times dy) = \nu(dy) \}$, 
and for every coupling $\pi$ we use the analogous notations $\pi_{1:t}$, $\pi_{x_{1:t},y_{1:t}}$, $\pi^s_{x_{1:t},y_{1:t}}$, $\pi_{x_{1:0}, y_{1:0}}$.

\subsection{Robustness}
\textbf{Optimal stopping:} Let $f\colon \R^{d}\times \{1,\ldots,T\} \to \R$ bounded and consider the optimal stopping problem:
\[
v(\mu) = \sup_{\tau \in \calT}\E_{X\sim\mu}[f(X_\tau,\tau)],
\]
where $\calT$ is the set of bounded optimal stopping time adapted to the natural filtration; see \cite[Proposition~4.4]{acciaio2020causal} and \cite[Theorem~2.8]{bartl2023sensitivity}.

\textbf{Utility maximization:} Let $\ell\colon\mathbb{R}\to\mathbb{R}$ be a convex loss function, i.e., $\ell$ is bounded from below and convex. Moreover let $g\colon\mathbb{R}^T\to\mathbb{R}$ be (the negative of) a payoff function and consider the problem:
\[ v(\mu) \coloneqq \inf_{a\in\mathcal{A}} \E_{X\sim\mu}\Big[ \ell\Big( g(X) + \sum_{t=1}^{T-1} a_{t+1}(X)(X_{t+1}-X_{t}) \Big) \Big],\]
where $\mathcal{A}$ denotes the set of all \emph{predictable controls} bounded by $L$, i.e., every $a=(a_t)_{t=1}^T\in\mathcal{A}$ is such that $a_t\colon\mathbb{R}^T\to\mathbb{R}$ only depends on $x_{1:t-1}$ and that $|a_t|\leq B$ for a fixed constant $B$; see \cite[Proposition~4.8]{acciaio2020causal} and \cite[Corollary~2.7]{bartl2023sensitivity}.

\textbf{Risk minimization:} Let $\avar_{\alpha}^\mu(\cdot)$ be the Average Value at Risk (or Expected Shortfall) at confidence level $\alpha > 0$ under $\mu$, and consider the problem:
\[
v(\mu) \coloneqq \inf_{a\in\mathcal{A}}\avar_{\alpha}^\mu\Big(\sum_{t=1}^{T-1} a_{t+1}(X)(X_{t+1}-X_{t}) \Big),
\]
where $\mathcal A$ is defined as above; see \cite[Corollary~6.9]{pflug2014multistage} and \cite[Corollary~3.4]{acciaio2024time}.

\subsection{Proofs.}
\label{subsec:proof}
\begin{proof}[Proof of Proposition~\ref{prop.dpp}]
The proof follows directly from \cite[Proposition 5.2]{backhoff2017causal} by taking the cost function $c(x,y) = \sum_{t=1}^T\|x_t-y_t\|^2$.
\end{proof}
\begin{proof}[Proof of Theorem~\ref{thm:stat}]
    Since $\mu \in \calP_2(\R^{dT})$ implies $\mu \in \calP_1(\R^{dT})$, then by Theorem~2.7 in \cite{acciaio2024convergence}, we prove the almost sure convergence.
\end{proof}

\subsection{Closed-form for Gaussians}
\label{subsec:close_form_gaussian}
For Gaussian distributions, closed-form have been established for both $\AW_2$-distance and $\W_2$-distance.
\begin{theorem}[Theorem~2.2 in \cite{takatsu2011wasserstein}]
\label{thm:W}
Let $\mu = \mathcal N(a,A)$ and $\nu = \mathcal N(b,B)$ be non-degenerate Gaussians on $\R^T$. Then
\begin{equation} 
\label{eq:W}
    \W^2_{2}(\mu, \nu) = |a - b |^2 + \tr\Big(A+B-2\big(A^{\frac{1}{2}}B A^{\frac{1}{2}}\big)^{\frac{1}{2}} \Big).
\end{equation}
\end{theorem}

\begin{theorem}[Theorem~1.1 in \cite{gunasingam2025adapted}] \label{thm:AW}
    Let $\mu = \mathcal N(a,A)$ and $\nu = \mathcal N(b,B)$ be non-degenerate Gaussians on $\R^T$, whose covariance matrices have Cholesky decompositions $A = LL^\top$ and $B = MM^\top$. Then
    \begin{equation}
        \label{eq:AW}
        \AW^2_{2}(\mu,\nu) = |a - b|^2 +
        \tr\big(LL^\top + MM^\top - 2 |M^\top L|\big). 
    \end{equation}
\end{theorem}
For $d$-dimensional Gaussian processes with $T$ time steps and $d > 1$, the $\W_2$-distance follows the same closed form as $\W_2$-distance is regardless of time structure, while the close-form of $\AW_2$-distance is slightly different from the $d=1$ case.
\begin{theorem}[Theorem~2.2 in \cite{acciaio2024entropic}]
\label{thm.AW_dT}
Let $\mu = \mathcal N(a,A)$ and $\nu = \mathcal N(b,B)$ be non-degenerate Gaussians on $\R^{dT}$, whose covariance matrices have Cholesky decompositions $A = LL^\top$ and $B = MM^\top$. Then
\begin{equation}
\label{eq:AW_dT}
    \begin{split}
        \AW_{2}^2(\mu,\nu)
        &= |a-b|^2 + \tr(LL^\top + MM^\top) - 2\tr(S),
    \end{split}
\end{equation}
where $S = \diag(S_1,\dots,S_T)$, with $S_{t}$ being the diagonal matrix of singular values of $(M^\top L)_{t,t}$.
\end{theorem}
When the marginal $\mu$ is degenerate, the Cholesky decomposition is not unique. Different choice of Cholesky decomposition leads to different $\AW$-distances. This might sounds weird at the first glance. However, if you see different Cholesky decomposition characterizing different filtration structure, this totally make sense for stochastic process. For example, $L = \begin{bmatrix}
    0 & 0 \\ 1 & 1\\
\end{bmatrix}$ and $M = \begin{bmatrix}
    0 & 0 \\ 0 & \sqrt{2}\\
\end{bmatrix}$ lead to the same Gaussian distribution because $LL^\top = MM^\top$. However, from a stochastic process point of view, they are indeed different. Let $G \sim \calN(0,I_2)$ be a noise process and define two processes: $X = LG$ and $Y = MG$. Let $\mathcal F_t = \sigma(G_{1:t}), t=1,2$ the natural filtration generated by the noise process. Then $Y_2 = \sqrt{2}G_2$ is independent of $\mathcal F_1$ while this is not the case for $X_2 = G_1 + G_2$. This argument is formally addressed by extending $\AW$-distance to \emph{filtered processes} $\mathcal F \mathcal P$ in \cite{bartl2021WSSP}, where $(\mathcal F \mathcal P, \AW)$ builds a geodesic space, which is isometric to a classical Wasserstein space. Under the filtered process setting, Gaussian process are characterized uniquely by the Cholesky decomposition instead of covariance matrix, and by following the same proof in \cite{acciaio2024entropic}, the closed-form of \eqref{eq:AW} and \eqref{eq:AW_dT} generalizes to degenerate Gaussian distributions.

\end{document}